\documentclass{article}

\usepackage[preprint]{neurips_2026}
\usepackage{graphicx}
\usepackage{subcaption}
\usepackage[T1]{fontenc}
\usepackage[
    colorlinks=true,
    linkcolor=blue,
    citecolor=blue,
    urlcolor=blue,
    filecolor=blue,
    bookmarks=true,
    bookmarksopen=true,
    bookmarksnumbered=true,
    pdfstartview=FitH,
    breaklinks=true
]{hyperref}
\usepackage{microtype}
\newcommand{\fdiPanel}[1]{%
    \begin{minipage}[t]{\linewidth}
        \vspace{0pt}\centering\includegraphics[width=\linewidth]{#1.pdf}
    \end{minipage}}
\newcommand{\fdiLegend}{%
    \par\vspace{2pt}{\scriptsize\includegraphics{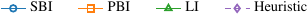}}}

\usepackage{amsmath}
\usepackage{amssymb}
\usepackage{amsthm}
\usepackage[ruled,lined,linesnumbered]{algorithm2e}

\newcommand{\rl}{r_{\min}}
\newcommand{\ru}{r_{\max}}

\usepackage[capitalize,noabbrev]{cleveref}

\theoremstyle{plain}
\newtheorem{theorem}{Theorem}[section]
\newtheorem{lemma}[theorem]{Lemma}
\newtheorem{corollary}[theorem]{Corollary}

\newtheorem{definition}[theorem]{Definition}
\newtheorem{assumption}[theorem]{Assumption}

\crefname{assumption}{assumption}{assumptions}
\Crefname{assumption}{Assumption}{Assumptions}

\theoremstyle{remark}
\newtheorem{remark}[theorem]{Remark}

\title{Practical Adversarial Attacks on Stochastic Bandits via Fake Data Injection}

\author{%
{\bfseries
Qirun Zeng$^{1}$ \quad
Eric He$^{2}$ \quad
Richard Hoffmann$^{2}$ \quad
Xuchuang Wang$^{3}$ \quad
Jinhang Zuo$^{1}$\thanks{Corresponding author.}} \\[0.7em]
{\small
$^{1}$City University of Hong Kong, Hong Kong SAR, China} \\[-0.05em]
{\small
$^{2}$California Institute of Technology, Pasadena, CA, USA} \\[-0.05em]
{\small
$^{3}$Hong Kong Baptist University, Hong Kong SAR, China} \\[0.4em]
{\small
\texttt{qirun.zeng@my.cityu.edu.hk}
\quad
\texttt{\{ehe, rhoffman\}@caltech.edu}
} \\[-0.05em]
{\small
\texttt{xuchuangw@hkbu.edu.hk}
\quad
\texttt{jinhang.zuo@cityu.edu.hk}}
}

\begin{document}

\maketitle

\begin{abstract}
Adversarial attacks on stochastic bandits have traditionally relied on some unrealistic assumptions, such as per-round reward manipulation and unbounded perturbations, limiting their relevance to real-world systems.
We propose a more practical threat model, Fake Data Injection, which reflects realistic adversarial constraints: the attacker can inject only a limited number of bounded fake feedback samples into the learner's history, simulating legitimate interactions.
We design effective attack strategies under this model, explicitly addressing both magnitude constraints (on reward values) and temporal constraints (on when and how often data can be injected).
Our theoretical analysis shows that these attacks can mislead a class of bandit algorithms into selecting a target arm in nearly all rounds while incurring only sublinear attack cost. Experiments on synthetic and real-world datasets validate the effectiveness of our strategies, revealing vulnerabilities in stochastic bandit algorithms under practical adversarial scenarios.
\end{abstract}

\section{Introduction}\label{sec:intro}

Multi-armed bandit (MAB) algorithms are widely used in online decision-making systems for their ability to balance exploration and exploitation using partial feedback. They form the backbone of many interactive applications, including personalized recommendation \citep{li2010contextual,zeng2026best,zou2026conversational}, online advertising \citep{chen2016combinatorial}, clinical trials \citep{villar2015multi}, and adaptive routing \citep{li2016contextual}. As these algorithms are increasingly deployed in high-stakes, user-facing systems, growing concerns have emerged regarding their vulnerability to adversarial manipulation.
A growing body of work \citep{jun2018adversarial, liu2019data, zuo2023adversarial, zuo2024near} has shown that MAB algorithms are vulnerable to adversarial attacks, where an attacker subtly perturbs the observed rewards to mislead the learning process. Remarkably, even with limited intervention, the attacker can steer the learner toward selecting a targeted but suboptimal arm in the vast majority of rounds.

However, prior works on adversarial attacks against stochastic bandits \citep{jun2018adversarial, liu2019data, zuo2024near} almost exclusively adopt a \emph{feedback-perturbation} threat model, where the attacker observes the learner's chosen arm and the corresponding reward in each round, then arbitrarily modifies that reward before it is revealed to the learner. This assumption effectively grants the attacker \emph{full control of the feedback channel at every timestep}, a level of adversarial power that is unrealistic in real-world systems. For example, in a restaurant recommendation platform, no adversary can continuously overwrite genuine user ratings submitted by thousands of real customers. At most, they can register fake accounts and inject fabricated reviews. Likewise, in online advertising or click-through prediction, practical attacks take the form of click fraud or synthetic interactions, which add fake feedback rather than altering genuine user data. In other words, prior models \emph{overestimate} the attacker's per-round control, while simultaneously \emph{underestimating} their ability to inject feedback across arbitrary arms.

These mismatches motivate our shift to a more realistic and constrained threat model, \emph{Fake Data Injection}. Instead of modifying genuine feedback, the attacker influences the learner indirectly by injecting a limited number of fabricated (arm, reward) pairs into its interaction history. These fake samples must conform to valid feedback ranges (e.g., binary clicks or 1--5 star ratings), and their injection is subject to constraints such as system-level detection or resource limits. The learner processes these fake interactions indistinguishably from real ones, updating estimates, counts, and decision logic accordingly.

This new threat model captures practical attack surfaces overlooked by previous works. It removes the unrealistic assumption of per-round reward manipulation, enables feedback injection on arbitrary arms, and respects the bounded nature of real-world feedback. At the same time, it introduces new algorithmic challenges that cannot be addressed by existing techniques. Unlike the standard model where the attacker perturbs rewards in real time, fake data injection raises fundamental questions about \emph{when}, \emph{how strongly}, and \emph{how frequently} to inject samples to effectively influence the learner. The attacker must decide: 
(i) how many fake samples are required to suppress the selections of a non-target arm,
(ii) how to achieve this using only bounded reward values, and
(iii) how to distribute injections over time when batch size or injection frequency is constrained.
These challenges call for new analytical tools and attack strategies that explicitly account for both \emph{magnitude} constraints (on reward values) and \emph{temporal} constraints (on when and how often data can be injected).

\subsection{Our Contributions}

We develop a new framework for adversarial attacks on stochastic bandits under realistic fake-data constraints. Our main contributions are summarized as follows:

\begin{itemize}

\item We introduce \emph{Fake Data Injection}, a new threat model in which the attacker cannot overwrite genuine feedback and must instead inject fabricated but valid $(\text{arm}, \text{reward})$ records. This model explicitly captures practical constraints on both the magnitude of injected rewards and the timing/frequency of injections, resulting in a more realistic attack setting.

\item We develop the \emph{suppression principle} for fake data injection attacks, showing that once a non-target arm becomes sufficiently suppressed, the learner’s own exploration mechanism prevents further exploration of that arm. For UCB, we formalize this through an exponential suppression lemma and develop the \emph{Simultaneous Injection (SI)} attack, which achieves successful attacks using only bounded fake injections with sublinear cost.

\item We further study temporally constrained attacks where the attacker can inject only limited batches over time. We propose \emph{Periodic Bounded Injection (PBI)}, which maintains suppression through periodic injections and achieves successful attacks with sublinear cost under both temporal and magnitude constraints.

\item We extend the suppression analysis beyond UCB through a general \emph{suppression-bounded} framework for a class of bandit algorithms whose exploration of persistently suppressed arms is bounded, including Thompson Sampling (TS) and $\epsilon$-greedy.

\item We validate our methods on both synthetic and real-world datasets, demonstrating that even sparse, bounded fake data can significantly bias bandit learners in practice.

\end{itemize}

Our work bridges theoretical attack models for stochastic bandits and practical data-driven attacks in real systems by formalizing fake data injection under bounded and limited injections. Further analyses for TS and $\epsilon$-greedy are provided in Appendices \ref{app:ts-extensions} and \ref{app:eps-greedy}, respectively.

\subsection{Related Work}\label{sec:related-work}
There is a growing body of work on adversarial attacks against bandit algorithms \citep{jun2018adversarial, liu2019data, garcelon2020adversarial, ma2023adversarial, wang2022linear}. 
\citet{jun2018adversarial} showed that with sublinear cost, an attacker can steer UCB and $\epsilon$-greedy toward suboptimal arms. 
Building on this, \citet{liu2019data} extended the paradigm to unknown algorithms and contextual bandits, while \citet{garcelon2020adversarial} studied perturbations to both contexts and rewards in linear bandits. 
More recently, \citet{zuo2024near} introduced optimal attack strategies against UCB and TS, establishing matching lower bounds. Most of these works assume a strong \emph{feedback-perturbation} model with direct reward modifications each round \citep{jun2018adversarial, liu2019data, garcelon2020adversarial, zuo2024near}, or bounded variants that still allow continuous interference \citep{xu2021observation, wang2024adversarial, zuo2023adversarial}. \citet{zeng2026beyond} shows that attack can succeed even when the target arm's mean reward lies exactly at the lower endpoint of the reward range.

Another line of research has focused on defense algorithms against adversarial attacks and corruptions, including robust MAB \citep{lykouris2018stochastic, gupta2019better,liu2019data,mukherjee2021mean}, robust linear bandits \citep{bogunovic2021stochastic, he2022nearly}, and combinatorial bandits \citep{xu2021simple}. However, the corruption-based threat models considered in these works are typically weaker than those studied in the attack literature, since the adversary is not allowed to observe the learner's actions before corrupting the feedback.

We differ from both lines by introducing the Fake Data Injection threat model, where adversaries cannot tamper with genuine rewards but may inject fabricated (arm, reward) pairs. This setting reflects realistic constraints in recommendation or advertising platforms, where fake users can be created but authentic feedback remains immutable. Beyond bandits, it parallels data poisoning in broader ML systems (recommenders, ads, A/B testing, RL), and our tools (e.g., the Exponential Suppression Lemma) provide general insights into vulnerabilities of sequential decision-making.

 \section{Preliminaries}\label{sec:model}

\subsection{Stochastic Bandits}
We consider the standard stochastic multi-armed bandit setting with the arm set $\mathcal K \triangleq \{1, 2, \dots, K\}$, where each arm $i \in \mathcal K$ is associated with an unknown reward distribution with mean $\mu_i$. The reward distributions are assumed to be $\sigma^2$-sub-Gaussian, with $\sigma^2$ known. 
In each round $t \in \mathcal T \triangleq \{1, 2, \ldots, T\}$, the learner selects an arm $a_t$ and receives a reward $r_t$ drawn from the distribution corresponding to arm $a_t$.

\textbf{Upper Confidence Bound (UCB)}. 
We consider the UCB algorithm as specified in \citet{jun2018adversarial,zuo2024near}, and the prototype is the $(\alpha, \psi)$ algorithm of~\citet[Section~2.2]{bubeck2012regret}. In the first $K$ rounds, the learner pulls each arm once to initialize reward estimates. For subsequent rounds $t > K$, the learner selects the arm with the highest UCB index $a_t = \arg\max_{i \in \mathcal K} \Lambda_i (t) \triangleq \hat \mu_i(t) + 3\sigma \sqrt{\frac{\log t}{N_i(t)}}$, where $\hat{\mu}_i(t)$ is the empirical mean reward of arm $i$ and $N_i(t)$ is the number of times arm $i$ has been selected up to round $t$.

\subsection{Previous Threat Model and Limitations}

Prior work on adversarial robustness in stochastic bandits broadly falls into two categories: corruption-robust approaches (e.g., \citet{kapoor2019corruption,gupta2019better}) and adversarial attack models (e.g., \cite{jun2018adversarial,zuo2023adversarial,zuo2024near}). The former assumes a weaker, history-only adversary and focuses on designing defenses against such corruption, whereas the latter studies a stronger, adaptive adversary that is standard in adversarial-attack settings. Our paper fits into this adversarial-attack category. Therefore, we begin by reviewing the standard threat model adopted in prior works on adversarial attack models.

In each round $t$, the learner selects an arm $a_t$ to play, and the environment generates a pre-attack reward $r^0_t$ drawn from the underlying distribution of arm $a_t$. The attacker then observes the tuple $(a_t, r^0_t)$ and decides an attack value $\alpha_t$. The learner receives only the post-attack reward $r_t = r^0_t - \alpha_t$. Define the cumulative attack cost as $C(T) = \sum_{t=1}^T |\alpha_t|$. The attacker's objective is to manipulate the learner into selecting a specific target arm for a linear number of rounds while incurring only sublinear attack cost. Formally, an attack is successful if it induces the learner to select the target arm in $T-o(T)$ rounds while incurring sublinear attack cost $C(T)=o(T)$. Although many attack strategies have been proposed under the standard threat model, it exhibits several critical limitations when applied to practical settings.

\emph{First}, the model assumes that the attacker can perturb the environment-generated reward in \textit{every} round. This assumption is often unrealistic in real-world applications such as recommender systems. For example, consider an app that recommends restaurants to users and collects their feedback to improve future recommendations. An attacker may wish to bias the system toward recommending a particular restaurant, but it is infeasible to directly modify the feedback of all real users. In practice, a more common attack strategy is to create fake users who submit fabricated feedback. However, even this is constrained by operational or detection limits: adding fake users in every round is highly impractical. Thus, the assumption of per-round attack capability does not reflect realistic adversarial power.
\emph{Second}, the model restricts the attacker to modifying only the reward of the chosen arm in each round. In contrast, fake-user-based attacks offer more flexibility. A fake user can submit feedback on any item (i.e., any arm), regardless of what the learner selected in that round. This means fake data injection enables the attacker to fabricate feedback for arbitrary arms, not just the one currently played by the learner. As a result, the standard model underestimates the attacker's flexibility in practice and overestimates their ability to act at every timestep.
\emph{Third}, the threat model assumes that both the pre-attack reward and the attack values are unbounded. In many systems, user feedback is naturally bounded (e.g., binary click signals or discrete rating scores). Allowing arbitrarily large attack values could result in out-of-range or clearly invalid feedback, which would either be filtered out by the system or easily flagged as suspicious. Therefore, attacks that rely on large reward perturbations are incompatible with these bounded-feedback environments.

\section{New Threat Model: Fake Data Injection}\label{sec:new-model}

To address practical limitations of the standard adversarial attack model, we introduce a new and more realistic threat model, which we call the Fake Data Injection Threat Model. This model captures how adversaries behave in real-world systems such as recommendation platforms, where direct manipulation of genuine user feedback is infeasible, and attacks are often carried out by injecting fabricated interactions (e.g., fake users with fake feedback).

In the Fake Data Injection model, the attacker does \emph{not} interfere with the feedback received by the learner during normal interactions. Instead, at round $t$, it may initiate an injection batch $\mathcal B_t$, a finite multiset of fake samples $(a^F,r^F)$ satisfying $a^F\in\mathcal K$ and $r^F\in [\rl, \ru]$. The learner processes these samples one by one as valid records: each fake sample updates the corresponding empirical mean, increments the arm's sample count, and advances the round counter by one.

Let $N_i^0(t)$ and $N_i^F(t)$ denote the numbers of genuine and fake samples, respectively, of arm $i$ up to round $t$. The learner's sample count is $N_i(t)=N_i^0(t)+N_i^F(t)$.
The total number of fake samples injected up to round $t$ is $N^F(t) \triangleq \sum_{i \in \mathcal K} N_i^F(t)$. Enumerate the fake samples injected by horizon $T$, in order of arrival, as $(a_s^F, r_s^F)$, $s = 1, \ldots, N^F(T)$. We measure attack cost by the number of injected samples $N^F(T)$ and the cumulative reward deviation
\begin{equation}\label{eq:injection-cost-measures}
    C^F(T) \triangleq \sum_{s=1}^{N^F(T)} \left|r_s^F-\mu_{a_s^F}\right|.
\end{equation}
Here, $N^F(T)$ captures the operational cost of injection, while $C^F(T)$ facilitates comparison with prior bandit-attack literature. We consider an attack \emph{successful} if it induces $T-o(T)$ genuine pulls of the target arm while ensuring both $N^F(T)=o(T)$ and $C^F(T)=o(T)$.

This hybrid mechanism distinguishes our model from both prior online perturbation models and offline poisoning settings. In offline data poisoning (e.g., \citet{liu2019data}), the attacker modifies a fixed dataset collected from arms actually played, and the learner is then trained once on this corrupted dataset. In contrast, our attacker can inject arbitrary $(i,r)$ pairs at any time, affecting arms that may never have been pulled. The learner immediately incorporates these fake samples and continues to make decisions in an online, round-by-round fashion. Thus, while our model permits batch-style updates, it remains fundamentally online, bridging the gap between classical reward-perturbation attacks and purely offline poisoning. It resolves several key limitations of the previous threat model:

\textbf{Limited access manipulation.} Unlike the standard model which assumes the attacker can modify the reward in every round, our model reflects the more plausible scenario where the attacker can only inject a limited number of fake interactions. For instance, in a restaurant recommendation app, an attacker cannot tamper with the feedback from real users but can register a finite number of fake accounts to submit biased reviews. It is unrealistic to assume the attacker can do this in every round without detection or resource exhaustion.

\textbf{Flexible feedback across arms.} The standard model restricts the attacker to modifying the reward for the arm chosen by the learner. In contrast, our model allows the attacker to fabricate data for any arm. This mirrors real-world attacks where fake users can submit reviews or feedback on arbitrary items, not just those recommended to them. For example, an attacker aiming to boost a target restaurant can flood the system with positive feedback for that restaurant, regardless of whether it was actually recommended in a specific round.

\textbf{Bounded and plausible feedback.} In our model, the fake rewards must satisfy $\rl\le r^F\le\ru$, consistent with many practical systems that collect binary clicks or scaled ratings (e.g., 1 to 5 stars). This avoids the unrealistic assumption of unbounded reward modifications, where a single large perturbation could dominate the learner's behavior. Our bounded injection design ensures the fake data remains indistinguishable from legitimate interactions.
\section{Attack Strategies}\label{sec:alg}

The central difficulty in fake data injection is that the attacker cannot erase genuine feedback or assign arbitrary corrupted rewards; it can only add valid samples to the learner's history. We address this constraint through a \emph{suppression principle}: inject enough bounded fake feedback so that each non-target arm's empirical estimate falls below a stable lower benchmark for the target arm. Once this separation is certified, the learner's own exploration rule prevents the suppressed arm from being selected often, turning targeted manipulation into a problem of establishing and maintaining non-target suppression.
This perspective is the main technical driver of our attacks. For UCB, we show that suppression can persist for an exponentially long interval, which leads to Simultaneous Injection (SI) and Periodic Bounded Injection (PBI) attacks with sublinear cost under different injection constraints. We then abstract the same idea into a suppression-bounded framework that applies beyond UCB. We label the target arm as $K$ and assume it has the lowest expected reward.\footnote{This is the most challenging target for the attacker; the same arguments extend directly to any other target arm.} We first analyze UCB and then extend the suppression-based argument to a broader class of algorithms.

\subsection{The Suppression Lemma}
Before presenting our concrete attack strategies, we establish the suppression mechanism used throughout the section.

\begin{lemma}[Exponential suppression under UCB]
\label{lem:exponential-suppression-ucb}
For the UCB algorithm specified in~\Cref{sec:model}, let $\delta \in (0, \frac{1}{2})$, and define $\hat{\ell}_K(t) \triangleq \hat{\mu}_K(t) - 2\beta(N_K(t)) - 3\sigma \delta_0$ with $\beta(N) \triangleq \sqrt{\frac{2\sigma^2}{N}\log \frac{\pi^2 K N^2}{3\delta}}$.
If
\begin{equation}
\hat{\mu}_i(t) \le \hat{\ell}_K(t), \label{eq:suppress-num}
\end{equation}
then arm $i$ is not selected again until at least round $\exp\!\big(N_i(t)\delta_0^2\big)$, with probability at least $1 - \delta$.
\end{lemma}

\begin{remark}\label{rem:exponential-suppression}
    \Cref{lem:exponential-suppression-ucb} establishes a critical property of our attack strategy: once a non-target arm $i$ has been pulled sufficiently and a properly chosen fake data injection is applied, its UCB index becomes \textit{exponentially suppressed}. More precisely, if the following two conditions are satisfied: (1) $N_i(t) \ge \frac{\log T}{\delta_0^2}$ and (2) $\hat{\mu}_i(t) \leq \hat{\ell}_K(t)$, then arm $i$ will not be selected again until after round $T$.
    
    This suppression effect is crucial: it guarantees that once the attack is applied to arm $i$, its influence on the learning process becomes negligible for the remaining rounds. The attacker can thus prevent further exploration of non-target arms using only a \emph{simultaneous (or sequential) injection} per arm, ensuring that the learner increasingly concentrates on the target arm. This mechanism forms the backbone of all our attack strategies.
\end{remark}

\subsection{Warm-up: Simultaneous Injection}
\label{sec:si}

We begin our study of fake data injection attacks by considering a setting in which the attacker is nearly free of discovery and can inject fake data points as needed. We formalize this insight through the Simultaneous Injection Algorithm, a simple yet effective attack strategy against the UCB algorithm, as shown in \Cref{alg:si-ucb}.

The attack operates as follows. For each non-target arm $i$, we wait until it has been pulled $N_i(t) = \lceil \log T / \delta_0^2\rceil$ times. At this point, we inject a batch of fake samples designed to reduce its empirical mean below a high-probability lower bound of the target arm. The injected samples for non-target arm $i$ ensure that after the attack, the empirical mean of arm $i$ satisfies $\hat{\mu}_i (t+n_i) \le \hat{\ell}_K (t)$,
thus making it unlikely to be selected in future rounds. 
We further require that the lower bound on bounded fake rewards is not too large,
formalized by the assumption below.
\begin{assumption}\label{ass:bounded-reward-condition}
    $\rl \le \mu_K - 3\beta(1) - 3\sigma\delta_0$.
\end{assumption}

This assumption ensures that suppression of non-target arms remains feasible even when injected rewards are bounded from below. In particular, it guarantees that the attacker can reduce the empirical mean of any non-target arm below the lower confidence bound of the target arm, despite the limited manipulation strength. In contrast, most prior works (e.g., \citet{jun2018adversarial, liu2019data, zuo2024near}) implicitly assume unbounded attacks, under which rewards can be decreased arbitrarily and such a margin condition holds automatically.
In practice, \Cref{ass:bounded-reward-condition} can be relaxed to the round-dependent condition $\rl < \hat{\mu}_K(t) - 2\beta(N_K(t)) - \eta$, for any fixed constant $\eta>0$, which only needs to hold at the specific round $t$ when the injection is carried out. For simplicity of exposition, we assume that \Cref{ass:bounded-reward-condition} holds throughout the subsequent analysis.
We now state the full guarantee for SI; its proof is given in Appendix~\ref{app:proof-si-ucb}.

\begin{algorithm}[t]
\caption{Simultaneous Injection (SI) on UCB}
\label{alg:si-ucb}

\KwIn{Attack parameter $\delta_0 > 0$}

\For{$t = 1, 2, \dots$}{
    \For{each non-target arm $i < K$}{
        \If{$N_i(t) = \lceil \log T / \delta_0^2 \rceil$}{
            Inject $n_i = \big\lceil 
            \frac{\hat{\mu}_i(t) - \hat{\ell}_K(t)}{\hat{\ell}_K(t) - \rl} 
            \cdot \lceil \tfrac{\log T}{\delta_0^2} \rceil
            \big\rceil$ samples $(i, \rl)$\;
        }
    }
}
\end{algorithm}

\begin{theorem}\label{thm:si-ucb}
    Suppose $T>2K$, $\delta\in(0,1/2)$, and $\rl<\mu_K-3\beta(1)-3\sigma\delta_0$. Define $\Delta_i\triangleq\mu_i-\mu_K$. With probability at least $1-\delta$, SI forces UCB to select the target arm in at least
    \[
        T - \mathcal{O}\left(
            \sum_{i = 1}^{K-1} \frac{\mu_i + \beta(\log T / \delta_0^2) - \rl}{\mu_K - 3\beta(1) - 3 \sigma \delta_0 - \rl} \cdot \frac{\log T}{\delta_0^2}
        \right)
    \]
    rounds. For each non-target arm $i<K$, the number of fake samples satisfies
    \[
        N_i^F(T)=\mathcal O\!\left(\frac{\Delta_i+4\beta(1)+3\sigma\delta_0}
        {\mu_K-3\beta(1)-3\sigma\delta_0-\rl}
        \cdot\frac{\log T}{\delta_0^2}\right).
    \]
    Moreover, the total injection count and reward-deviation cost satisfy
    \[
        N^F(T) = \sum_{i=1}^{K-1}N_i^F(T)
        \quad\text{and}\quad
        C^F(T) = \sum_{i=1}^{K-1}(\mu_i-\rl)N_i^F(T).
    \]
\end{theorem}
Compared with the attack algorithm under the standard threat model in \cite{jun2018adversarial}, the SI Algorithm achieves a similar level of target-arm selection with comparable sublinear attack cost. 
Notably, the parameter $\delta_0$ controls the trade-off between the number of non-target arm pulls and the attack cost: increasing $\delta_0$ reduces the number of non-target pulls but increases the cost per injection. However, the marginal benefit diminishes once $\delta_0 > \sqrt{\log T}$, beyond which the cost grows without improving effectiveness. 

\begin{remark}
\label{disc:varying-injected-rewards}Using the minimum reward $\rl$ simplifies the analysis and minimizes the number of injected samples, but identical fake rewards are not necessary. The attacker may instead use varying rewards whose batch average satisfies $\bar r^F<\hat\ell_K(t)$, replacing $\rl$ by $\bar r^F$ when computing the injection size. This allows more natural-looking reward mixtures, although a higher injected mean requires more samples to achieve the same suppression.   
\end{remark}

\subsection{Periodic Bounded Injection}

The Simultaneous Injection algorithm above assumes that the attacker can inject all required fake samples within a sequence of rounds. However, this assumption may not hold in practice. For example, in a restaurant recommendation system, injecting a large batch of fake (e.g., low-rating) reviews at once may trigger anomaly detection mechanisms, leading the system to filter or ignore the fake data. In contrast, injecting smaller amounts of fake feedback periodically at a controlled rate can be significantly less suspicious and more effective in practice.

To model this scenario, we introduce a more restrictive and realistic setting where:

\begin{enumerate}
    \item There exists a lower bound $\rl$ that the fake rewards cannot fall below;
    \item Each injection batch contains at most $f$ fake samples per arm;
    \item Consecutive injections on the same arm $i$ must be separated by a cooldown period of at least $R_i$ rounds, where $R_i$ depends on the maximum batch size $f$.
\end{enumerate}

To address this setting, we propose the Periodic Bounded Injection (PBI) algorithm, shown in \Cref{alg:pbi-ucb}. Given a maximum batch size $f$, the algorithm adaptively schedules periodic injections to suppress the empirical mean of non-target arms while respecting both constraints.

\begin{algorithm}[t]
\caption{Periodic Bounded Injection on UCB}
\label{alg:pbi-ucb}

\KwIn{Attack parameter $\delta_0$, reward lower bound $\rl$, max batch size $f$}
Initialize $n_i=0$ for every non-target arm $i$\;
\For{$t = 1, 2, \dots$}{
    \For{each non-target arm $i < K$}{
        \If{$N_i(t) = \lceil {\log T}/{\delta_0^2} \rceil$}{
            $n_i \gets 
            \big\lceil 
            \frac{\hat{\mu}_i(t) - \hat{\ell}_K(t)}{\hat{\ell}_K(t) - \rl} 
            \cdot \lceil \tfrac{\log T}{\delta_0^2} \rceil
            \big\rceil$;
            $\tau_i \gets t$\;
            Calculate $R_i$ according to~\Cref{eq:sleep-gap}\;
        }
        \If{$N_i^F(t)<n_i$ and $\tau_i \le t$}{
            Inject $\min\{n_i-N_i^F(t),f\}$ fake samples $(i,\rl)$\;
            $\tau_i \gets \tau_i + f + R_i$\;
        }
    }
}
\end{algorithm}

The PBI algorithm distributes the injection of fake samples across multiple rounds rather than injecting them all at once. Once a non-target arm $i$ reaches the designated pull threshold ($\lceil \log T / \delta_0^2 \rceil$), the attacker computes both the total number of fake samples $n_i$ required to suppress the empirical mean of arm $i$, and a waiting interval $R_i$, which ensures that the fake samples can be injected periodically without allowing arm $i$ to regain a high UCB index. The notation $\tilde \mu_i(t+f)$ represents the estimated value of $\hat \mu_i(t+f)$ with $f$ fake data injections starting from round $t$. At each interval of $R_i + f$ rounds, a batch of at most $f$ fake samples is injected until the total $n_i$ is exhausted. This strategy effectively balances stealthiness and attack \emph{efficacy}, making it robust against detection in practical systems with bounded feedback and rate-limited injection constraints.

\begin{lemma}
    \label{lem:pbi-suppression}
    Let $t_i(c) = t + (f + R_i)c$ denote the round before the $(c\!+\!1)$-th injection. And let
    \begin{equation}
    R_i \triangleq \min_{1 \le c \le \lceil n_i/f \rceil} 
        \frac{1}{c} \left[ \exp\!\left( \left( \frac{\hat{\mu}_K(t) - 2\beta(N_K(t)) - \tilde{\mu}_i(t_i(c))}{3\sigma} \right)^2 \cdot (N_i(t) + fc) \right) - t\right] - f\label{eq:sleep-gap}.
    \end{equation}
    The choice of $R_i$ ensures that once a batch of $f$ fake data samples is injected into non-target arm $i$, the arm will not be selected again for at least the next $R_i$ rounds.
\end{lemma}

\begin{proof}[Proof Sketch]
    The argument follows a one-shot variant of \Cref{lem:exponential-suppression-ucb}.
    It suffices to consider the first batch of $f$ injected samples, which induces the weakest suppression. 
    This batch creates a sufficient gap between the UCB index of arm $i$ and that of the target arm $K$. 
    By the definition of $R_i$, this gap persists for at least $R_i$ rounds, during which arm $i$'s UCB index remains below that of arm $K$. 
    Hence arm $i$ is not selected in this interval.
\end{proof}

Based on our experimental observations, setting $c=1$ generally leads to the smallest practical value of $R_i(c)$. The derivation of $R_i$ and the PBI proofs are given in Appendix~\ref{app:proof-pbi-ucb}.

\begin{theorem}\label{thm:pbi-ucb}
    Suppose $T>2K$, $\delta\in(0,1/2)$, and $\rl<\mu_K-3\beta(1)-3\sigma\delta_0$. Define $\Delta_i \triangleq \mu_i - \mu_K$. With probability at least $1-\delta$, PBI forces UCB to select the target arm in at least
    \[
        T - \mathcal O \left(\sum_{i=1}^{K-1} \frac{\mu_i + \beta (\log T / \delta_0^2) - \rl}{\mu_K - 3\beta(1) - 3\sigma\delta_0 - \rl} \cdot \frac{\log T}{\delta_0^2}\right)
    \]
    rounds. For each non-target arm $i<K$, the number of fake samples satisfies
    \[
        N_i^F(T)=\mathcal O\!\left(\frac{\Delta_i+4\beta(1)+3\sigma\delta_0}
        {\mu_K-3\beta(1)-3\sigma\delta_0-\rl}
        \cdot\frac{\log T}{\delta_0^2}\right).
    \]
    Moreover, the total injection count and reward-deviation cost satisfy
    \[
        N^F(T) = \sum_{i=1}^{K-1} N_i^F(T)
        \quad\text{and}\quad
        C^F(T) = \sum_{i=1}^{K-1}(\mu_i-\rl)N_i^F(T).
    \]
\end{theorem}
The suppression idea underlying SI can also be implemented by waiting for a non-target arm to be selected and applying a sufficiently large reward perturbation~\citep{jun2018adversarial,zuo2024near}. However, this connection does not directly extend to PBI: splitting injections into bounded batches changes arm counts and UCB indices over time, so the learner may revisit a non-target arm before all injections are completed. Moreover, reward perturbation is restricted to the currently selected arm and cannot freely realize PBI's schedule.

\subsection{Extension to Suppression-Bounded Algorithms}
\label{sec:general-suppression-bounded}

We now extend the suppression principle beyond UCB to a broader class of bandit algorithms whose exploration of persistently suppressed arms remains bounded. Intuitively, once an arm becomes empirically dominated by another arm, the algorithm can pull that arm only a bounded number of additional times.

\begin{definition}[Suppression-bounded algorithms]
\label{def:suppression-bounded-algorithm}
Fix a horizon $T$, a confidence level $\delta \in (0, 1)$. A bandit algorithm $\mathcal A$ is said to be \emph{suppression-bounded} if, for any arm $i$, there exists a deterministic bound $B_i^{\mathcal A}(t, T, \delta)$ such that the following implication holds with probability at least $1 - \delta$: if \[
    \hat \mu_i(s) < \max_{j \in \mathcal K} \hat \mu_j (s), \qquad \text{for all } s \in \{t + 1, \ldots, T\},
\]
then \[
    N_i(T) - N_i(t) \le B_i^{\mathcal A}(t, T, \delta).
\]
We refer to $B_i^{\mathcal A}(t,T,\delta)$ as the \emph{post-suppression pull bound} for arm $i$.
\end{definition}

This class covers many standard bandit algorithms, including UCB, Thompson Sampling, and $\epsilon$-greedy. Further analyses for Thompson Sampling and $\epsilon$-greedy appear in Appendices~\ref{app:ts-extensions} and~\ref{app:eps-greedy}; concrete post-suppression bounds are given in Appendix~\ref{app:post-suppression-instantiations}.

\begin{lemma}[The Suppression Lemma]
\label{lem:general-suppression}
Fix a non-target arm $i < K$ and a round $t$. Define
\[
    U_i(t) \triangleq \hat \mu_i^0(t) + 2\beta(N_i^0(t)),
    \qquad
    L_{K}(t) \triangleq \hat \mu_{K}^0(t) - 2\beta(N_{K}^0(t)).
\]
Suppose that before round $t$, the attacker injects $n_i$ fake samples into arm $i$ with $\rl < L_{K}(t)$, and injects no further fake samples into arm $i$ afterward. If
\begin{equation}
\label{eq:general-suppression-condition}
    n_i > \frac{U_i(t)-L_{K}(t)}{L_{K}(t) - \rl}
    \left(N_i^0(t) + B_i(t, T, \delta / K) \right),
\end{equation}
then, with probability at least $1-\delta-\delta/K$, arm $i$ is selected at most $B_i(t,T,\delta/K)$ additional times after the injection and before round $T$. Consequently, after applying the condition to every non-target arm, all non-target arms are suppressed simultaneously with probability at least $1-2\delta$.
\end{lemma}

\begin{theorem}
\label{thm:general-attack-cost}
Suppose $\mathcal A$ is suppression-bounded. Fix $\delta\in(0,1/2)$ and assume $\rl < \mu_K - 3 \beta(1)$. Define $\Delta_i \triangleq\mu_i-\mu_K$. For each non-target arm $i<K$, let $t_i$ be the round at which the attacker first suppresses arm $i$ relative to arm $K$. Before round $t_i$, the attacker injects the minimum number of fake samples satisfying \Cref{eq:general-suppression-condition}. Then, with probability at least $1 - 2\delta$, the learner selects the target arm in at least
\[
    T-\mathcal O \left(\sum_{i=1}^{K-1}
        \frac{\mu_i+3\beta(1)-\rl}{\mu_K-3\beta(1) - \rl}
        \left(N_i^0(t_i) + B_i \left(t_i, T, \delta / K \right)\right)
    \right)
\]
rounds. For each non-target arm $i<K$, the number of fake samples satisfies
    \[
        N_i^F(T)=\mathcal O\!\left(\frac{\Delta_i+6\beta(1)}{\mu_K-3\beta(1)-\rl}
        \left(N_i^0(t_i) + B_i \left(t_i, T, \delta / K \right)\right)\right).
    \]
    Moreover, the total injection count and reward-deviation cost satisfy
    \[
        N^F(T) = \sum_{i=1}^{K-1}N_i^F (T)
        \quad\text{and}\quad
        C^F(T) = \sum_{i=1}^{K-1}(\mu_i-\rl) N_i^F(T).
    \]
\end{theorem}

\Cref{lem:general-suppression} and \Cref{thm:general-attack-cost} show that any algorithm in which each non-target arm is pulled only $o(T)$ times can be suppressed at an attack cost of only $o(T)$, provided that fake samples are injected at a suitable time. This is consistent with~\citet[Fact 1]{zuo2024near}.
The guarantee is broadly applicable, but it is generally not as sharp as the algorithm-specific analyses for UCB or Thompson Sampling. For example, in the UCB case, the attacker can inject fake samples so that, for each non-target arm $i$,
    $\hat\mu_i(t) \le \hat\ell_K(t) < L_K(t)$,
which eliminates post-suppression pulls and effectively yields $B_i(t,T,\delta)=0$. This highlights the trade-off between generality and sharpness: the abstract theorem applies broadly once a valid post-suppression pull bound is available, whereas algorithm-specific arguments can exploit the precise decision rule to obtain stronger guarantees.
Proofs are given in Appendix~\ref{app:general-proofs}.

The theorem applies directly to the SI algorithm. For the PBI algorithm, one can adapt the argument by invoking the post-suppression pull bound $B_i(t,T,\delta)$ separately on each suppression period $R_i$. Aggregating these period-wise bounds yields the corresponding guarantee for PBI.

\begin{remark}
Our suppression-bounded framework is related to the mean-based algorithms studied by \citet{xu2021observation}, but differs in its defining condition. While mean-basedness imposes structural restrictions on the decision policy, suppression-boundedness directly bounds the number of subsequent pulls under persistent empirical domination. Both frameworks cover common algorithms such as UCB, Thompson Sampling, and $\epsilon$-greedy.
\end{remark}

\begin{remark}
Our analysis assumes genuine feedback is observed only for the selected arm. However, additional unsolicited clean feedback may dilute injected samples and undermine suppression. If such observations are bounded by $Q_i(T)$, they can be accommodated by replacing $B_i^{\mathcal A}(t,T,\delta)$ with $B_i^{\mathcal A}(t,T,\delta)+Q_i(T)$ in \Cref{eq:general-suppression-condition}. Thus, sublinear arrivals preserve sublinear attack cost, whereas linear arrivals may require linear cost. This dilution mechanism is similar in spirit to the smoothed myopic response defense of~\citet{zuo2024near}, although it incurs $\Theta(T)$ regret even without attack.
\end{remark}

\section{Experiments}\label{sec:exp}
\begin{figure*}[htbp]
    \centering
    \begin{subfigure}[t]{0.48\linewidth}
        \vspace{0pt}
        \centering
        \fdiPanel{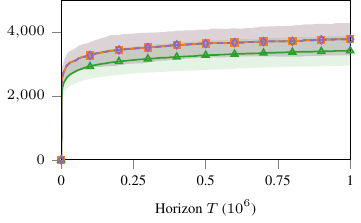}
        \caption{Attack cost vs.\ $T$}
        \label{fig:ucb-cost-horizon}
    \end{subfigure}
    \hfill
    \begin{subfigure}[t]{0.48\linewidth}
        \vspace{0pt}
        \centering
        \fdiPanel{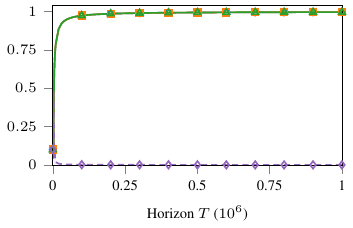}
        \caption{Target-arm ratio vs.\ $T$}
        \label{fig:ucb-selection-ratio}
    \end{subfigure}
    \fdiLegend
    \caption{UCB experiments with 20 repetitions. Shading shows one sample standard deviation.}
    \label{fig:ucb}
\end{figure*}

We compare SBI (the bounded SI attack), PBI, LI, and a random-timing heuristic. LI uses one fake sample per triggered non-target arm, with its reward chosen to meet the suppression threshold. The heuristic uses SBI's fake rewards and per-arm sample counts but randomizes injection times, giving the same cost.
We set $K=10$, $\delta=0.05$, $\ru=1$, and $f=10$. Methods share genuine reward streams and initialize each arm once; $T$ counts genuine and fake records. We vary $T$ from $K$ to $10^6$; both panels use the same 20 runs at each $T$. PBI withholds uncertified injections, and all trials are included.
For UCB, we sample 10 movies from a MovieLens-derived binary rating matrix~\cite{harper2015movielens} and generate Bernoulli rewards using their mean ratings, targeting the least popular movie. We set $\sigma=0.5$ and $\delta_0=0.2$. Thompson Sampling settings and results appear in Appendix~\ref{app:ts-experiments}.
\Cref{fig:ucb} shows attack cost $C^F(T)$ and target-arm selection. At $T=10^6$, SBI and PBI reach $99.68\%$, versus $0.02\%$ for the matched-cost heuristic. All SBI/PBI schedules are feasible on this grid, with no non-target pulls after each arm's first injection. Code and raw results are available at GitHub\footnote{\href{https://github.com/qirunzeng/bandit-fake-data-injection}{https://github.com/qirunzeng/bandit-fake-data-injection}}.

\section{Concluding Remarks}\label{sec:remarks}

This work introduces \emph{Fake Data Injection}, a practical threat model for adversarial attacks on stochastic bandits. Unlike prior models that allow per-round, unbounded reward perturbations, our framework reflects realistic constraints, including bounded feedback, limited injection capability, and the attacker's inability to modify genuine user data. Under this model, we develop effective attack strategies that manipulate bandit learners using only sublinear-cost injections. Both theory and experiments show that even sparse, bounded fake interactions can significantly bias stochastic bandit algorithms.
Our analytical tools, such as the suppression lemma, further reveal a broader vulnerability pattern in sequential decision-making systems. These results expose a gap between classical robustness guarantees and the practical security of online learning under realistic feedback constraints.

Despite these results, several limitations still remain. We study a passive learner that processes all feedback without defense, whereas robust systems may use filtering or auditing. We focus on stochastic bandits; extending the framework to contextual bandits or reinforcement learning remains open. In addition, real-world attackers may face detection risks or adaptive filtering, which our current model does not capture. Our model also abstracts away user identities, collaborative filtering, and richer user behavior. Future work could investigate defenses such as anomaly detection and arm-level auditing, explore the dynamic interaction between attackers and adaptive learners, and extend the analysis to settings where multiple users can inject action–reward pairs.

\section*{Acknowledgement}
The work of Jinhang Zuo is supported by the Early Career Scheme (No. CityU 21212025) of the University Grants Committee (UGC) and CityU 9610706.
The work of Xuchuang Wang is supported in part by the RGC JRFS2526-4S06.

\bibliography{refs}
\bibliographystyle{plainnat}

\appendix

\clearpage
\section*{Appendix}
\phantomsection
\begingroup
\small
\tableofcontents
\endgroup

\clearpage
\section{Proofs for UCB}

For notational simplicity, we treat all numbers here as \emph{integers}; this convention does not affect the results.

Suppose that the reward distributions of arms are $\sigma^2$-sub-Gaussian. The following concentration result will be useful throughout our analysis.
Recall that $\beta(N) \triangleq \sqrt{\frac{2\sigma^2}{N} \log \frac{\pi^2 K N^2}{3\delta}}$.
Let
\begin{equation} \label{eq:event-e}
    \mathcal E \triangleq \left\{\forall i,t,|\hat\mu^0_i(t) - \mu_i|<\beta(N_i^0(t))\right\},
\end{equation}
where $\hat{\mu}^0_i(t)$ and $N_i^0(t)$ denote the empirical mean and count of genuine samples of arm $i$ up to round $t$, respectively.

\begin{lemma}[Lemma 1 of \citet{jun2018adversarial}]
    \label{lem:concentration-e}
    For $\delta \in (0,1)$, $\mathbb{P}\{\mathcal E\} > 1 - \delta$.
\end{lemma}

\subsection{Proof of the suppression lemma for UCB}

\begin{proof}[Proof of \Cref{lem:exponential-suppression-ucb}]
    Under event $\mathcal{E}$, we can establish a lower bound on arm $K$'s estimate. For any two rounds $t_2 > t_1$,
    \begin{equation}
        \begin{aligned}
            \hat \mu_K (t_2) \geq \mu_K - \beta (N_K(t_2)) \geq \mu_K - \beta (N_K(t_1)) \geq \hat \mu_K (t_1) - 2\beta(N_K(t_1)),
        \end{aligned}\label{eq:target-mean-lower-bound}
    \end{equation}
where the second inequality follows from the monotonicity of $\beta(N)$.

Suppose that at round $t_1$, the following equation holds,
\begin{equation}
    \hat{\mu}_i(t_1) \le \hat{\mu}_K(t_1) - 2 \beta(N_K(t_1)) - 3\sigma\delta_0. \label{eq:attack-condition}
\end{equation}
which means $\text{UCB}_i(t_1) < \text{UCB}_K(t_1)$, so arm $i$ is not selected at round $t_1$.

Now consider any subsequent round $t_2$ with $t_1 \le t_2 < \exp(m_i \delta_0^2)$, where $m_i \triangleq N_i(t_1)$, and assume that arm $i$ has not been selected in any round between $t_1$ and $t_2$. Then $N_i(t_2\!+\!1) = m_i$ and $\hat{\mu}_i(t_2\!+\!1) = \hat{\mu}_i(t_1)$, so the UCB index for arm $i$ at round $t_2 + 1$ is
\begin{align*}
    \text{UCB}_i(t_2 + 1) 
        &= \hat{\mu}_i(t_1) + 3\sigma\sqrt{\tfrac{\log(t_2+1)}{m_i}} \le \hat{\mu}_K(t_1) - 2\beta(N_K(t_1)) - 3\sigma\delta_0 + 3 \sigma \sqrt{\tfrac{\log(t_2+1)}{m_i}} \\
        &\le \hat{\mu}_K(t_1) - 2\beta(N_K(t_1)) \le \hat{\mu}_K(t_2 + 1) \le \text{UCB}_K(t_2 + 1),
\end{align*}
where the third step uses the bound $\sqrt{\frac{\log(t_2+1)}{m_i}} < \delta_0$. This argument shows that the UCB index of arm $i$ remains strictly lower than that of the target arm $K$ for all $t_2 < \exp(m_i \delta_0^2)$. By induction, arm $i$ will not be selected again before round $\exp(m_i \delta_0^2)$.
\end{proof}

\subsection{Proof of Simultaneous Injection on UCB}\label{app:proof-si-ucb}

\begin{proof}\label{proof:si-ucb}
According to \Cref{alg:si-ucb}, after injecting $n_i$ samples with value $\rl$ into arm $i$ at round $t$, the empirical mean at round $t + n_i$ becomes: \begin{align*}
    \hat{\mu}_i(t + n_i) 
    &= \frac{ \hat{\mu}_i(t) N_i(t) + n_i \cdot \rl }{ N_i(t) + n_i } \le \frac{ \hat{\mu}_i(t) N_i(t) + \frac{ \hat{\mu}_i(t) - \hat{\ell}_K(t) }{ \hat{\ell}_K(t) - \rl } N_i(t) \cdot \rl }{ N_i(t) + \frac{ \hat{\mu}_i(t) - \hat{\ell}_K(t) }{ \hat{\ell}_K(t) - \rl } N_i(t) } \\
    &= \frac{ \hat{\mu}_i(t) + \frac{ \hat{\mu}_i(t) - \hat{\ell}_K(t) }{ \hat{\ell}_K(t) - \rl } \rl }{ 1 + \frac{ \hat{\mu}_i(t) - \hat{\ell}_K(t) }{ \hat{\ell}_K(t) - \rl } } = \frac{ \hat{\mu}_i(t) \hat{\ell}_K(t) - \hat{\ell}_K(t) \rl }{ \hat{\mu}_i(t) - \rl } = \hat{\ell}_K(t),
\end{align*}
where the last step ensures that after injection, arm $i$'s empirical mean is suppressed below the target threshold.

On $\mathcal E$, we first bound the number of injected samples. At the trigger time $t_i$, concentration and the positive margin imply
\begin{align*}
    n_i &= \frac{\hat\mu_i(t_i)-\hat\ell_K(t_i)}{\hat \ell_K(t_i)-\rl}N_i(t_i) \le \frac{\Delta_i + \beta(N_i(t_i))+3\beta(N_K(t_i))+3\sigma\delta_0}
    {\mu_K-3\beta(N_K(t_i))-3\sigma\delta_0-\rl}\cdot\frac{\log T}{\delta_0^2} \\
    &\le
    \frac{\Delta_i+4\beta(1)+3\sigma\delta_0}
    {\mu_K-3\beta(1)-3\sigma\delta_0-\rl}\cdot\frac{\log T}{\delta_0^2}.
\end{align*}
Consequently,
\[
    N^F(T)=\sum_{i=1}^{K-1}N_i^F(T) \le \sum_{i=1}^{K-1} n_i
    \le\sum_{i=1}^{K-1}
    \frac{\Delta_i+4\beta(1)+3\sigma\delta_0}
    {\mu_K-3\beta(1)-3\sigma\delta_0-\rl}\cdot\frac{\log T}{\delta_0^2},
\]

Next, since event $\mathcal{E}$ holds, and by \Cref{lem:exponential-suppression-ucb}, arm $i$ will not be selected after the injection. The total number of history entries for arm $i$, including fake samples, is therefore bounded by:
\begin{align*}
    N_i(t) + n_i 
    &\le \frac{\log T}{\delta_0^2} + \frac{ \hat{\mu}_i(t) - \hat{\ell}_K(t) }{ \hat{\ell}_K(t) - \rl } \cdot \frac{\log T}{\delta_0^2} = \frac{ \hat{\mu}_i(t) - \rl }{ \hat{\ell}_K(t) - \rl } \cdot \frac{\log T}{\delta_0^2}  \\
    &\le \frac{ \mu_i + \beta(N_i(t)) - \rl }{ \hat{\mu}_K(t) - 2 \beta(N_K(t)) - 3\sigma\delta_0 - \rl } \cdot \frac{\log T}{\delta_0^2}
    \le \frac{
        \mu_i + \beta (\frac{\log T}{\delta_0^2}) - \rl 
    }{ 
        \mu_K - 3\beta(1) - 3\sigma\delta_0 - \rl
    } \cdot \frac{\log T}{\delta_0^2},
\end{align*}
where we used concentration bounds for $\hat{\mu}_i(t)$ and $\hat{\mu}_K(t)$ under event $\mathcal{E}$ and the nonincreasing property of $\beta(\cdot)$.
The cumulative reward-deviation cost is:
\[
    C^F(T)
    = \sum_{i=1}^{K-1} (\mu_i-\rl)N_i^F(T) \le \sum_{i=1}^{K-1} (\mu_i-\rl)
        \frac{\Delta_i + 4\beta(1) + 3\sigma \delta_0}{\mu_K - 3\beta(1) - 3\sigma\delta_0 - \rl}
    \frac{\log T}{\delta_0^2}.
\]

Therefore, both the cumulative attack cost and the number of non-target arm pulls are $\mathcal{O}\left( \tfrac{\log T}{\delta_0^2} \right)$ per arm, and hence sublinear in $T$, completing the proof.
\end{proof}

\subsection{Proofs of Periodic Bounded Injection on UCB}\label{app:proof-pbi-ucb}

\begin{proof}[Proof of \Cref{lem:pbi-suppression}]
Suppose event $\mathcal{E}$ holds. We begin by estimating the total number of fake samples needed to demote arm $i$. This quantity is given by
\[
    n_i = \frac{\hat{\mu}_i(t) - \hat{\ell}_K(t)}{\hat{\ell}_K(t) - \rl} \cdot \frac{\log T}{\delta_0^2},
\]
where $\hat{\ell}_K(t)$ is a conservative lower bound on arm $K$'s empirical mean, and $\rl$ is the value of each injected fake sample. Under the batch size constraint $f$, the attack spans $\lceil \frac{n_i}{f} \rceil$ periods, with each period injecting $f$ fake samples.

Our goal is to choose an appropriate delay parameter $R_i$ such that after each injection, arm $i$ is not selected again until the next scheduled injection. Specifically, we require that after the $c$-th batch (for any $c \in \{1, 2, \dots, \lceil \frac{n_i}{f} \rceil\}$), arm $i$ is not selected for at least $R_i$ rounds.

Let $t_i(c) = t + (f + R_i) c$ denote the round before the $(c+1)$-th injection. We examine the UCB index of arm $i$ at time $t_i(c)$:
\begin{equation}
    \label{eq:ucb-i}
    \text{UCB}_i(t_i(c)) = \hat{\mu}_i(t_i(c)) + 3 \sigma \sqrt{ \frac{\log t_i(c)}{N_i(t) + f c} } = \frac{N_i(t) \hat{\mu}_i(t) + f c \rl}{N_i(t) + f c} + 3\sigma \sqrt{ \frac{\log t_i(c)}{N_i(t) + f c} }.
\end{equation}

To ensure arm $i$ is not selected before round $t_i(c)$, we want its UCB index to be no larger than that of arm $K$. A sufficient condition is
\begin{align}
    \label{eq:ucb-condition}
    \text{UCB}_i(t_i(c)) \le \hat{\mu}_K(t) - 2 \beta(N_K(t)) \le \text{UCB}_K(t_i(c)),
\end{align}
where we use the lower bound $\hat{\mu}_K(t) - 2\beta(N_K(t))$ to conservatively approximate arm $K$'s UCB.

Define $\tilde{\mu}_i(t_i(c))$ as the post-injection empirical mean of arm $i$: \[
    \tilde{\mu}_i(t_i(c)) = \frac{N_i(t) \hat{\mu}_i(t) + f c \rl}{N_i(t) + f c}.
\]
Then, the condition in \Cref{eq:ucb-condition} implies that, for any $c$, the delay parameter $R_i(c)$ must satisfy:
\begin{align}
    \label{eq:r-i-c}
    R_i(c) \le \frac{\exp\left( \left( \frac{ \hat{\mu}_K(t) - 2\beta(N_K(t)) - \tilde{\mu}_i(t_i(c))}{3\sigma} \right)^2 \cdot (N_i(t) + f c) \right) - t}{c} - f.
\end{align}

Finally, to ensure that this condition holds for every injection period, we define the overall delay parameter as the minimum over all $c$:
\begin{align}
    \label{eq:r-i-final}
    R_i = \min_{1 \le c \le \left\lceil \frac{n_i}{f} \right\rceil}  \frac{\exp\left( \left( \frac{ \hat{\mu}_K(t) - 2\beta(N_K(t)) - \tilde{\mu}_i(t_i(c))}{3\sigma} \right)^2 \cdot (N_i(t) + f c) \right) - t}{c} - f.
\end{align}

This choice of $R_i$ ensures that after each batch of $f$ fake samples, arm $i$ will not be pulled again until the next scheduled injection.
\end{proof}

\begin{proof}[Proof of \Cref{thm:pbi-ucb}]
Consider any non-target arm $i < K$. Under PBI, periodic injections are initiated once $N_i(t) = \frac{\log T}{\delta_0^2}$. The total number of fake samples required, denoted by $n_i$, is identical to that in SI and is chosen such that the empirical mean of arm $i$ is driven below the lower confidence bound of the target arm $K$.

Unlike SI, the PBI algorithm distributes these $n_i$ injections over time in batches of size $f$, separated by delays of length $R_i$. By \Cref{lem:pbi-suppression}, after each batch of $f$ injected samples, arm $i$ will not be selected for at least $R_i$ rounds. This establishes a persistent suppression effect: throughout the entire injection process, arm $i$ is not revisited by the learner, and therefore no additional genuine samples of arm $i$ are collected before all $n_i$ fake samples are injected.
We now consider two cases:

Case 1: All $n_i$ fake samples are successfully injected before round $T$. Then the empirical mean of arm $i$ is fully suppressed, and its future selection is prevented. This directly mirrors the SI case analyzed in \Cref{thm:si-ucb}, so both the number of non-target arm pulls and cumulative attack cost are bounded as in that theorem.

Case 2: Only a portion of the $n_i$ fake samples are injected by round $T$. Since \Cref{lem:pbi-suppression} guarantees that arm $i$ is not pulled during this partial injection phase, the number of target-arm pulls remains high. Meanwhile, both the number of injected samples and the cumulative attack cost are no larger than in the SI case.

As a consequence, the two cases both incur no more than
$$
n_i = \frac{\hat \mu_i(t)-\hat \ell_K(t)}{\hat \ell_K(t)-\rl} \cdot \frac{\log T}{\delta_0^2}
$$
fake samples.

We first bound the number of injected samples. The prescribed $n_i$ obeys the per-arm bound established in the proof of \Cref{thm:si-ucb}, so
\[
    N^F(T)=\sum_{i=1}^{K-1}N_i^F(T)\le\sum_{i=1}^{K-1}n_i
    \le\sum_{i=1}^{K-1}
    \frac{\Delta_i+4\beta(1)+3\sigma\delta_0}
    {\mu_K-3\beta(1)-3\sigma\delta_0-\rl}\cdot\frac{\log T}{\delta_0^2}.
\]

Next, since no further genuine pulls of arm $i$ occur during injection, its total contribution to the learner's history is bounded by $N_i(t)+n_i$, where
\begin{align*}
    N_i(t) + n_i 
    &\le \frac{\log T}{\delta_0^2}
    + \frac{ \hat{\mu}_i(t)-\hat{\ell}_K(t)}{ \hat{\ell}_K(t)-\rl} \cdot \frac{\log T}{\delta_0^2} \\
    &= \frac{ \hat{\mu}_i(t)-\rl}{ \hat{\ell}_K(t)-\rl} \cdot \frac{\log T}{\delta_0^2} \\
    &\le \frac{ \mu_i+\beta(N_i(t))-\rl}{ \hat{\mu}_K(t)-2\beta(N_K(t))-3\sigma\delta_0-\rl} \cdot \frac{\log T}{\delta_0^2} \\
    &\le \frac{ \mu_i+\beta\!\left(\frac{\log T}{\delta_0^2} \right)-\rl}{ \mu_K-3\beta(1)-3\sigma\delta_0-\rl} \cdot \frac{\log T}{\delta_0^2}.
\end{align*}

The cumulative reward-deviation cost is:
\[
    C^F(T) = \sum_{i=1}^{K-1}(\mu_i-\rl)N_i^F(T) \le\sum_{i=1}^{K-1}(\mu_i-\rl)
    \frac{\Delta_i +4\beta(1)+3\sigma\delta_0}
    {\mu_K-3\beta(1)-3\sigma\delta_0-\rl}
    \frac{\log T}{\delta_0^2}.
\]
This proves the claim.
\end{proof}

\section{Thompson Sampling Algorithm-specific Analysis}
\label{app:ts-extensions}

\textbf{Thompson Sampling (TS)}.
We consider the Thompson Sampling algorithm specified in \citet{zuo2024near}, and the prototype is the $(\alpha, \psi)$ algorithm of \citet{agrawal2017near}. In the first $K$ rounds, the learner pulls each arm once. For rounds $t > K$, a sample $\nu_i$ is drawn independently for each arm $i \in \mathcal K$ from the distribution $\mathcal{N} \left(\hat{\mu}_i(t), \frac{1}{N_i(t)}\right)$, and the learner selects the arm with the largest sampled value $a_t = \arg\max_{i \in \mathcal K} \nu_i$.
In the absence of attacks, both UCB and TS are known to achieve sublinear regret by selecting suboptimal arms only $o(T)$ times.

For the Thompson Sampling extensions, we also use an event $\mathcal F$ that controls the sampled value $\nu_i$ of each arm:
\begin{equation} 
\label{eq:event-f}
    \mathcal F \triangleq \left\{\forall  i, t, |\nu_i(t) - \hat \mu_i(t)| < \frac{\gamma(t)}{\sqrt{N_i(t)}}\right\},
\end{equation}
where $\gamma(t) \triangleq \sqrt{2\log \frac{\pi^2 K t^2}{ 3\delta}}$.

\begin{lemma}\label{lem:concentration-f}
    For $\delta \in (0,1)$, $\mathbb{P}\{\mathcal F\} > 1 - \delta$.
\end{lemma}

\begin{proof}
Since $\nu_i(t)\sim \mathcal{N} \left(\hat{\mu}_i(t), \frac{1}{N_i(t)}\right)$, it follows that $\nu_i(t) - \hat{\mu}_i(t) \sim \mathcal{N}\!\left(0, \frac{1}{N_i(t)}\right)$.
By the Gaussian tail bound, for any $c > 0$,
\[
\Pr\!\left(\left|\nu_i(t)-\hat{\mu}_i(t)\right| > c\right)
< 2 \exp\!\left(-\frac{c^2}{2 / N_i(t)}\right)
= 2 \exp\!\left(-\frac{N_i(t) c^2}{2}\right).
\]
Setting
$c = \frac{\gamma(t)}{\sqrt{N_i(t)}}$ and 
$\gamma(t) = \sqrt{2 \log \frac{\pi^2 K t^2}{3\delta}}$,
we obtain
\[
\Pr\!\left(\left|\nu_i(t)-\hat{\mu}_i(t)\right| > \frac{\gamma(t)}{\sqrt{N_i(t)}}\right)
< 2 \exp\!\left(-\frac{\gamma(t)^2}{2}\right)
= \frac{6\delta}{\pi^2 K t^2}.
\]
Summing over $t$ yields the claim.
\end{proof}

We collect the Thompson Sampling extensions here. The event $\mathcal F$ in \Cref{eq:event-f} controls the posterior samples around their empirical means; combined with the genuine-sample event $\mathcal E$, it yields a TS analogue of the UCB suppression lemma. We then state SI-TS and PBI-TS in the same order as the main text develops SI and PBI for UCB.

\subsection{Suppression Lemma for Thompson Sampling}

Define
\[
    \Gamma_i(t)
    \triangleq
    \sqrt{2\log \frac{\pi^2K}{3\delta}+4N_i(t)\delta_0^2},
    \qquad
    \kappa_i^{\mathrm{TS}}(t)
    \triangleq
    \Gamma_i(t)\left(\frac{1}{\sqrt{N_i(t)}}+\frac{1}{\sqrt{N_K(t)}}\right),
\]
and $\hat \ell'_K(t) \triangleq \hat \mu_K(t) - 2 \beta(N_K(t)) - \kappa_i^{\mathrm{TS}}(t)$.
The margin $\kappa_i^{\mathrm{TS}}(t)$ is the largest combined posterior-sampling deviation of arms $i$ and $K$ over the suppression window ending at $\exp(N_i(t)\delta_0^2)$.

\begin{lemma}[Exponential suppression under Thompson Sampling]
\label{lem:exponential-thompson}
For each non-target arm $i < K$, if $\hat{\mu}_i(t) \le \hat{\ell}'_K(t)$, then with probability at least $1 - 2\delta$, arm $i$ will not be selected again until at least round $\lfloor \exp(N_i(t) \delta_0^2) \rfloor$.
\end{lemma}

\begin{proof}
    Suppose that at round $t_1$, the following inequality holds:
    \begin{equation}
        \hat{\mu}_i(t_1) \le \hat{\mu}_K(t_1) - 2\beta(N_K(t_1)) - \kappa_i^{\mathrm{TS}}(t_1). \label{eq:ts-gap-condition}
    \end{equation}
    Let $m_i \triangleq N_i(t_1)$ and consider any round $t_2$ such that $t_1 < t_2 < \lfloor \exp(m_i \delta_0^2) \rfloor$. Assuming that arm $i$ is not selected from round $t_1$ to $t_2$, then $N_i(t_2+1) = m_i$ and $\hat{\mu}_i(t_2+1) = \hat{\mu}_i(t_1)$. Applying the concentration bounds from \Cref{lem:concentration-e,lem:concentration-f}, the sampled value $\nu_i(t_2 + 1)$ for arm $i$ satisfies:
    \begin{align*}
        \nu_i(t_2+1) 
        &< \hat{\mu}_i(t_2+1) + \frac{\gamma(t_2+1)}{\sqrt{m_i}} \le \hat{\mu}_K(t_1) - 2\beta(N_K(t_1))
            - \kappa_i^{\mathrm{TS}}(t_1)
            + \frac{\gamma(t_2+1)}{\sqrt{m_i}} \\
        &\le \hat{\mu}_K(t_2+1)
            - \kappa_i^{\mathrm{TS}}(t_1)
            + \frac{\gamma(t_2+1)}{\sqrt{m_i}}.
    \end{align*}
    Since $t_2+1\le \exp(m_i\delta_0^2)$, we have
    $\gamma(t_2+1)\le \Gamma_i(t_1)$. Also, $N_K(t_2+1)\ge N_K(t_1)$.
    Hence
    \[
        \kappa_i^{\mathrm{TS}}(t_1)
        \ge
        \gamma(t_2+1)
        \left(\frac{1}{\sqrt{m_i}}+\frac{1}{\sqrt{N_K(t_2+1)}}\right).
    \]
    Therefore
    \[
        \nu_i(t_2+1)
        \le
        \hat{\mu}_K(t_2+1)
        - \frac{\gamma(t_2+1)}{\sqrt{N_K(t_2+1)}}
        \le \nu_K(t_2+1),
    \]
    where the last inequality uses event $\mathcal F$ for arm $K$.

    This chain of inequalities implies that the sampled value $\nu_i(t_2+1)$ remains lower than $\nu_K(t_2+1)$ for all $t_2 < \exp(m_i \delta_0^2)$, with probability at least $1 - 2\delta$. Therefore, arm $i$ will not be selected again until at least round $\exp(m_i \delta_0^2)$.
\end{proof}

\subsection{Simultaneous Injection on Thompson Sampling}

\begin{algorithm}[ht]
\caption{SI on TS}
\label{alg:si-ts}

\KwIn{Attack parameter $\delta_0$, target arm $K$, lower bound $\rl$}

\For{$t = 1, 2, \dots$}{
    \For{each non-target arm $i < K$}{
        \If{$N_i(t) = \lceil \frac{\log T}{\delta_0^2} \rceil$}{
            Inject $n_i =
            \Big\lceil
                \frac{\hat{\mu}_i(t) - \hat{\ell}'_K(t)}
                     {\hat{\ell}'_K(t) - \rl}
                \cdot
                \Big\lceil \frac{\log T}{\delta_0^2} \Big\rceil
            \Big\rceil$ fake samples $(i, \rl)$\;
        }
    }
}
\end{algorithm}

\begin{theorem}\label{thm:si-ts}
    Suppose $T>2K$ and $\delta\in(0,1/2)$. Define $\Delta_i\triangleq\mu_i-\mu_K$. For each non-target arm $i$, let $t_i$ be the round at which \Cref{alg:si-ts} triggers the injection and set
    \[
    \bar \kappa_i^{\mathrm{TS}} 
    \triangleq 
    \sqrt{2\log \frac{\pi^2 K}{3\delta} + 4\log T} \left(1+{\sqrt{\frac{\delta_0^2}{\log T}}}\right) .
    \] 
    Assume $\mu_K-3\beta(1)-\bar\kappa_i^{\mathrm{TS}}-\rl>0$ for every non-target arm. With probability at least $1 - 2\delta$, the modified Simultaneous Injection attack forces Thompson Sampling to select the target arm in at least
    \begin{equation}
        T - \mathcal{O}\left(\sum_{i=1}^{K-1} \frac{\mu_i + \beta ({\log T}/{\delta_0^2}) - \rl}{\mu_K - 3\beta(1) - \bar\kappa_i^{\mathrm{TS}} - \rl} \cdot \frac{\log T}{\delta_0^2}\right)
    \end{equation}
    rounds. For each non-target arm $i<K$, the number of fake samples satisfies
    \[
        N_i^F(T) = \mathcal O \! \left( \frac{\Delta_i + 4\beta(1) + \bar \kappa_i^{\rm TS}}
        {\mu_K - 3\beta(1) - \bar \kappa_i^{\rm TS} - \rl}
        \cdot\frac{\log T}{\delta_0^2} \right).
    \]
    Moreover, the total injection count and reward-deviation cost satisfy
    \[
        N^F(T) = \sum_{i=1}^{K-1}N_i^F(T)
        \quad\text{and}\quad
        C^F(T)\le\sum_{i=1}^{K-1}(\mu_i-\rl)N_i^F(T).
    \]
\end{theorem}

\begin{proof}
    \label{proof:si-ts}
    According to \Cref{alg:si-ts}, after injecting $n_i$ samples, the empirical mean at round $t + {n}_i$ becomes:
    \begin{align*}
        \hat \mu_i(t+n_i) &= \frac{\hat \mu_i(t) N_i(t) + n_i \rl}{N_i(t) + n_i} \le \frac{\hat \mu_i(t) N_i(t) + \frac{\hat \mu_i (t) - \hat \ell^{'}_K(t)}{\hat \ell^{'}_K(t) - \rl} N_i(t)\rl}{N_i(t) + \frac{\hat \mu_i (t) - \hat \ell^{'}_K(t)}{\hat \ell^{'}_K(t) - \rl} N_i(t)} \\
        &= \frac{\hat \mu_i(t) + \frac{\hat \mu_i (t) - \hat \ell^{'}_K(t)}{\hat \ell^{'}_K(t) - \rl} \rl}{1 + \frac{\hat \mu_i (t) - \hat \ell^{'}_K(t)}{\hat \ell^{'}_K(t) - \rl}} = \frac{\hat \mu_i (t) \hat \ell^{'}_K(t) - \hat \ell^{'}_K(t) \rl}{\hat \mu_i(t) - \rl} = \hat \ell^{'}_K(t),
    \end{align*}
    where the last step ensures that after injection, arm $i$'s empirical mean is suppressed below the target threshold.

    On $\mathcal E\cap\mathcal F$, we first bound the number of injected samples. Applying concentration directly to the ratio defining $n_i$ gives
    \begin{align*}
        n_i
        &=\frac{\hat\mu_i(t_i)-\hat\ell'_K(t_i)}{\hat \ell'_K (t_i)-\rl}N_i(t_i) \le \frac{\Delta_i + \beta(N_i(t_i))+3\beta(N_K(t_i))+\kappa_i^{\mathrm{TS}}(t_i)}
        {\mu_K-3\beta(N_K(t_i))-\kappa_i^{\mathrm{TS}}(t_i)-\rl}\cdot\frac{\log T}{\delta_0^2} \\
        &\le
        \frac{\Delta_i+4\beta(1)+\bar\kappa_i^{\mathrm{TS}}}
        {\mu_K-3\beta(1)-\bar\kappa_i^{\mathrm{TS}}-\rl}\cdot\frac{\log T}{\delta_0^2}.
    \end{align*}
    Consequently,
    \[
        N^F(T)=\sum_{i=1}^{K-1}N_i^F(T)\le\sum_{i=1}^{K-1}n_i
        \le\sum_{i=1}^{K-1}
        \frac{\Delta_i+4\beta(1)+\bar\kappa_i^{\mathrm{TS}}}
        {\mu_K-3\beta(1)-\bar\kappa_i^{\mathrm{TS}}-\rl}\cdot\frac{\log T}{\delta_0^2},
    \]
    which gives the stated injection-count bound.

    Next, since events $\mathcal{E}$ and $\mathcal{F}$ hold, and by \Cref{lem:exponential-thompson}, arm $i$ will not be selected after round $t + {n}_i$. The total number of history entries for arm $i$, including fake samples, is therefore bounded by:
    
    \begin{align*}
        N_i(t) + n_i &= \frac{\log T}{\delta_0^2}
            + \frac{\hat \mu_{i}(t) - \hat \ell^{'}_K(t)}{\hat\ell^{'}_K(t) - \rl}
            \frac{\log T}{\delta_0^2} 
            = \frac{\hat \mu_{i}(t) - \rl}{\hat\ell^{'}_K(t) - \rl} \frac{\log T}{\delta_0^2} \\
        &\le \frac{\mu_{i} + \beta(N_i(t)) - \rl}{\hat \mu_K(t) - 2\beta(N_K(t)) - \kappa_i^{\mathrm{TS}}(t) - \rl} \frac{\log T}{\delta_0^2} \\
        &\le \frac{
            \mu_{i} + \beta(\frac{\log T}{\delta_0^2}) - \rl
        }{
            \mu_K - 3\beta(1) - \kappa_i^{\mathrm{TS}}(t) - \rl
        } \frac{\log T}{\delta_0^2}\\
        &\le \frac{
            \mu_{i} + \beta(\frac{\log T}{\delta_0^2}) - \rl
        }{
            \mu_K - 3\beta(1) - \bar\kappa_i^{\mathrm{TS}} - \rl
        } \frac{\log T}{\delta_0^2}.
    \end{align*}

    The cumulative reward-deviation cost is:
    \[
        C^F(T)
        =\sum_{i=1}^{K-1}(\mu_i-\rl)N_i^F(T) \le\sum_{i=1}^{K-1}(\mu_i - \rl)
        \frac{\Delta_i + 4\beta(1) + \bar\kappa_i^{\mathrm{TS}}}{\mu_K - 3\beta(1) - \bar\kappa_i^{\mathrm{TS}} - \rl}
        \frac{\log T}{\delta_0^2}.
    \]
    This proves the theorem.
\end{proof}

\subsection{Periodic Bounded Injection on Thompson Sampling}
Let $t_i(c) = t + (f + R_i) c$ denote the round before the $c\!+\!1$-th injection and \[
\eta_i(t) \triangleq \frac{
    \left(
        \hat{\mu}_K(t)
        - 2\beta(N_K(t))
        - \tilde{\mu}_i(t_i(c))
    \right)^2
}{
    2\left(
        1 / \sqrt{N_i(t)+fc}
        + 1 / \sqrt{N_K(t)}
    \right)^2
}
\]
\begin{algorithm}[t]
\caption{PBI on TS}
\label{alg:pbi-ts}
\KwIn{Attack parameter $\delta_0$, reward bound $[\rl, \ru]$, max batch size $f$}
Initialize $n_i=0$ for every non-target arm $i$\;
\For{$t = 1, 2, \dots$}{
\For{each non-target arm $i < K$}{
\If{$N_i(t) = \lceil \tfrac{\log T}{\delta_0^2} \rceil$}{
    $n_i \gets
    \Big\lceil
        \frac{\hat{\mu}_i(t) - \hat{\ell}'_K(t)}
             {\hat{\ell}'_K(t) - \rl}
        \cdot
        \Big\lceil \frac{\log T}{\delta_0^2} \Big\rceil
    \Big\rceil$;
    $\tau_i \gets t$\;
    $R_i \gets
    \min_{1 \le c \le \left\lceil \frac{n_i}{f} \right\rceil}
    \left\{
    \frac{1}{c}
    \left(
        \sqrt{
            \frac{3\delta}{\pi^2 K}
            \cdot
            \exp\!\left(
                \eta_i(t)
            \right)
        }
        - t
    \right)
    - f
    \right\}$\;
}
\If{$N_i^F(t)<n_i$ and $\tau_i \le t$}{
    Inject $\min\{n_i-N_i^F(t),f\}$ fake samples $(i,\rl)$\;
    $\tau_i \gets \tau_i + f + R_i$\;
}
}
}
\end{algorithm}

\begin{lemma}
\label{lem:pbi-suppression-thompson}
The choice of $R_i$ in \Cref{alg:pbi-ts} ensures that once a batch of $f$ fake data samples is injected into non-target arm $i$, the arm will not be selected again for at least the next $R_i$ rounds.
\end{lemma}

\begin{proof}
We aim to guarantee that arm $i$ is not selected between successive fake sample injections. Let $\tilde{\nu}_i$ and $\tilde{\nu}_K$ denote the Thompson-sampled values of arm $i$ and arm $K$, respectively, after the $c$-th injection. 
After injecting $f$ fake samples with value $\rl$ for $c$ periods, the empirical mean of arm $i$ becomes
\[
\tilde{\mu}_i(t_i(c)) = \frac{N_i(t) \hat{\mu}_i(t) + f c \rl}{N_i(t) + f c}.
\]
We want to ensure that arm $i$ is unlikely to be selected before time $t_i(c)$ by ensuring:
\begin{align}
    \label{eq:thompson-domination}
    \tilde{\nu}_i(t_i(c)) 
    &\le \tilde{\mu}_i(t_i(c))
        + \frac{\gamma(t_i(c))}{\sqrt{N_i(t)+fc}} \nonumber \\
    &\le \hat{\mu}_K(t) - 2\beta(N_K(t))
        - \frac{\gamma(t_i(c))}{\sqrt{N_K(t)}} 
    \le \tilde{\nu}_K(t_i(c)),
\end{align}
where $\gamma(t) = \sqrt{2 \log \left( \frac{\pi^2 K t^2}{3\delta} \right)}$ bounds the Thompson Sampling deviation under event $\mathcal{F}$.

Rearranging the middle inequality in \Cref{eq:thompson-domination}, we require:
\begin{align}
\label{eq:thompson-gap}
\gamma(t_i(c))
\left(1 / \sqrt{N_i(t)+fc} + 1 / \sqrt{N_K(t)}\right)
\le \hat{\mu}_K(t) - 2\beta(N_K(t)) - \tilde{\mu}_i(t_i(c)).
\end{align}

Solving \Cref{eq:thompson-gap} for $t_i(c)$ gives:
$t_i(c) \le \sqrt{ \frac{3\delta}{\pi^2 K} \cdot \exp\left( \eta_i(t) \right) }$.
Therefore, we define:
\begin{align}
\label{eq:r-i-thompson}
R_i = \min_{1 \le c \le \left\lceil \frac{n_i}{f} \right\rceil}
\frac{1}{c} \left[
\sqrt{ \frac{3\delta}{\pi^2 K} \cdot \exp \left( 
    \eta_i(t)
    \right) } - t \right]  - f.
\end{align}

This ensures that after each batch of $f$ fake samples, arm $i$ is suppressed for at least $R_i$ rounds under events $\mathcal{E}$ and $\mathcal{F}$. Hence, the learner will not select arm $i$ between consecutive injections.
\end{proof}

\begin{theorem}
    \label{thm:pbi-ts}
    Suppose $T>2K$ and $\delta\in(0,1/2)$. Define $\Delta_i\triangleq\mu_i-\mu_K$. For each non-target arm $i$, let $t_i$ be the round at which \Cref{alg:pbi-ts} triggers the injection schedule and set $\bar\kappa_i^{\mathrm{TS}}\triangleq\kappa_i^{\mathrm{TS}}(t_i)$. Assume $\mu_K-3\beta(1)-\bar\kappa_i^{\mathrm{TS}}-\rl>0$ for every non-target arm. With probability at least $1 - 2\delta$, the modified PBI forces Thompson Sampling to select the target arm in at least \begin{equation}
        T - \mathcal{O}\left(\sum_{i=1}^{K-1} \frac{\mu_i + \beta \left(\log T / \delta_0^2\right) - \rl}{\mu_K - 3\beta(1) -\bar\kappa_i^{\mathrm{TS}} - \rl} \cdot \frac{\log T}{\delta_0^2}\right)
    \end{equation} 
    rounds. For each non-target arm $i<K$, the number of fake samples satisfies
    \[
        N_i^F(T)=\mathcal O\!\left(\frac{\Delta_i+4\beta(1)+\bar\kappa_i^{\mathrm{TS}}}
        {\mu_K-3\beta(1)-\bar\kappa_i^{\mathrm{TS}}-\rl}
        \cdot\frac{\log T}{\delta_0^2}\right).
    \]
    Moreover, the total injection count and reward-deviation cost satisfy
    \[
        N^F(T) = \sum_{i=1}^{K-1}N_i^F(T)
        \quad\text{and}\quad
        C^F(T) = \sum_{i=1}^{K-1}(\mu_i-\rl)N_i^F(T).
    \]
\end{theorem}

\begin{proof}[Proof of \Cref{thm:pbi-ts}]
Consider any non-target arm $i<K$, and let $t=t_i$ be the round at which \Cref{alg:pbi-ts} triggers the injection schedule for this arm. Under events $\mathcal E$ and $\mathcal F$, which hold jointly with probability at least $1-2\delta$, the choice of $R_i$ in \Cref{lem:pbi-suppression-thompson} guarantees that after each batch of fake samples is injected, arm $i$ is not selected before the next scheduled batch. Hence, during the entire injection process, no additional genuine samples of arm $i$ are collected.

The total number of fake samples prescribed by \Cref{alg:pbi-ts} is
\[
    n_i =
    \frac{\hat \mu_i(t)-\hat \ell'_K(t)}
         {\hat \ell'_K(t)-\rl}
    \cdot \frac{\log T}{\delta_0^2},
\]
up to the harmless integer-rounding convention stated at the beginning of the appendix. We now consider two cases.

If all $n_i$ fake samples are injected before round $T$, then the empirical mean of arm $i$ is reduced below $\hat\ell'_K(t)$. By \Cref{lem:exponential-thompson}, arm $i$ is then suppressed for the remainder of the horizon. If only a portion of the $n_i$ fake samples is injected before round $T$, then \Cref{lem:pbi-suppression-thompson} still ensures that arm $i$ is not selected during this partial-injection phase, and the incurred cost is no larger than the cost of injecting all $n_i$ samples. Thus, in both cases it is enough to upper bound the contribution of $N_i(t)+n_i$ samples in the learner's history.

We first bound the number of injected samples. Applying the same concentration argument as in the proof of \Cref{thm:si-ts} to the ratio defining $n_i$ yields
\[
    n_i\le\frac{\Delta_i+4\beta(1)+\bar\kappa_i^{\mathrm{TS}}}
    {\mu_K-3\beta(1)-\bar\kappa_i^{\mathrm{TS}}-\rl}\cdot\frac{\log T}{\delta_0^2}.
\]
Hence, even if a schedule is incomplete at $T$,
\[
    N^F(T)=\sum_{i=1}^{K-1}N_i^F(T)\le\sum_{i=1}^{K-1}n_i
    \le\sum_{i=1}^{K-1}
    \frac{\Delta_i+4\beta(1)+\bar\kappa_i^{\mathrm{TS}}}
    {\mu_K-3\beta(1)-\bar\kappa_i^{\mathrm{TS}}-\rl}\cdot\frac{\log T}{\delta_0^2},
\]
which gives the stated injection-count bound.

Next, since $N_i(t) = \frac{\log T}{\delta_0^2}$, we have
\begin{align*}
    N_i(t)+n_i
    &= \frac{\log T}{\delta_0^2}
    + \frac{\hat \mu_i(t)-\hat \ell'_K(t)}
           {\hat \ell'_K(t)-\rl}
      \frac{\log T}{\delta_0^2} = \frac{\hat \mu_i(t)-\rl}
            {\hat \ell'_K(t)-\rl}
       \frac{\log T}{\delta_0^2} \\
    &\le
    \frac{\mu_i+\beta(N_i(t))-\rl}
         {\hat \mu_K(t)-2\beta(N_K(t))-\kappa_i^{\mathrm{TS}}(t)-\rl}
    \frac{\log T}{\delta_0^2} 
    \le
    \frac{\mu_i+\beta(\frac{\log T}{\delta_0^2})-\rl}
         {\mu_K-3\beta(1)-\bar\kappa_i^{\mathrm{TS}}-\rl}
    \frac{\log T}{\delta_0^2}.
\end{align*}
Summing this bound over all non-target arms gives the stated lower bound on the number of target-arm selections.

The cumulative reward-deviation cost is:
\[
    C^F(T)
    =\sum_{i=1}^{K-1}(\mu_i-\rl)N_i^F(T) \le \sum_{i=1}^{K-1}(\mu_i-\rl)
    \frac{\Delta_i+4\beta(1)+\bar\kappa_i^{\mathrm{TS}}}
         {\mu_K-3\beta(1)-\bar\kappa_i^{\mathrm{TS}}-\rl}
    \frac{\log T}{\delta_0^2}.
\]
This proves the theorem.
\end{proof}

\section{Epsilon-greedy Algorithm-specific Analysis}\label{app:eps-greedy}

Here we provide the extension of \Cref{lem:general-suppression} on $\epsilon$-greedy algorithm.

\begin{corollary}[The $\epsilon$-greedy case]
\label{cor:general-suppression-eps-greedy}
Consider $\epsilon$-greedy with exploration rates $\{\epsilon_s\}_{s=1}^T$. Define
\[
B_i(t,T,\delta)\triangleq \sum_{s=t+1}^T \frac{\epsilon_s}{K}
+\sqrt{2\left(\sum_{s=t+1}^T \frac{\epsilon_s}{K}\right)\log\frac{1}{\delta}}
+\frac{1}{3}\log\frac{1}{\delta}.
\]
Then $\epsilon$-greedy has the post-suppression pull bound $B_i(t,T,\delta)$, and the condition in \Cref{eq:general-suppression-condition} becomes
\[
    n_i > \frac{U_i(t) - L_K(t)}{L_K(t) - \rl}
    \bigl(N_i^0(t) + B_i(t,T,\delta/K)\bigr).
\]
Under this condition, with probability at least $1-\delta-\frac{\delta}{K}$, $N_i(T)-N_i(t)\le B_i(t,T,\frac{\delta}{K})$.
Applying the condition to all non-target arms suppresses them simultaneously with probability at least $1-2\delta$.
\end{corollary}

\begin{proof}
If $\hat\mu_i(s)<\hat\mu_K(s)$ throughout rounds $t,\dots,T$, then the arm cannot be selected by the greedy part of $\epsilon$-greedy, so it can only be selected during exploration rounds. Let $X_s$ be the indicator that round $s$ is an exploration round and arm $i$ is chosen. Then $\{X_s\}_{s=t+1}^T$ are independent Bernoulli random variables with means $\epsilon_s/K$. Bernstein's inequality therefore gives
\[
    \sum_{s=t+1}^T X_s \le B_i(t,T,\delta/K)
\]
with probability at least $1-\frac{\delta}{K}$, so $B_i(t, T, \frac{\delta}{K})$ is a valid post-suppression pull bound for this algorithm at confidence level $\delta/K$.
Applying \Cref{lem:general-suppression} then yields the stated bound. The simultaneous guarantee follows by a union bound over the concentration event and the Bernstein events for all non-target arms.
\end{proof}

\section{General Suppression Framework Extensions}

\subsection{Concrete instantiations of the post-suppression bound}\label{app:post-suppression-instantiations}

We next instantiate the general suppression framework for UCB and Thompson Sampling.
For a simple instantiation, suppose that after the suppression step, throughout the remaining horizon the non-target arm has empirical mean at most $\mu_i^F$ and the target arm has empirical mean at least $\mu_K^F$, where $g_i^F \triangleq \mu_K^F-\mu_i^F>0$.
Since our attack injects fake samples only into non-target arms, and these samples can only decrease their empirical means, we may take $\mu_K^F=L_K(t_i)$ and $\mu_i^F = \hat \mu_i(t_i)+2\beta(N_i^0(t_i))$, with $\mu_i^F < \hat \mu_K (t_i) - 2 \beta(N_K(t_i)) = L_K(t_i)$ after suppression.

For UCB, if arm $i$ has count $N_i(s)$ at a later round $s\le T$, then its index is at most $\mu_i^F + 3 \sigma \sqrt{\frac{\log T}{N_i(s)}}$, whereas the target arm's UCB index is at least $\mu_K^F + 3\sigma\sqrt{\frac{\log t_i}{T}}$. Indeed, for every later round $s \ge t_i$, we have $\log s\ge \log t_i$ and $N_K(s) \le T$, so the target-arm exploration bonus is lower bounded by $3\sigma\sqrt{\frac{\log t_i}{T}}$. Hence arm $i$ cannot be selected once
\[
    N_i(s) \ge \frac{9\sigma^2\log T}{(g_i^F + 3\sigma\sqrt{{\log t_i}/{T}})^2}.
\]
Thus a valid post-suppression pull bound for the general framework is
\[
    B_i^{\mathrm{UCB}}(t,T,\delta/K)
    =
    \max\left\{
    \frac{9\sigma^2\log T}{(\mu_K^F-\mu_i^F+3\sigma\sqrt{{\log t_i}/{T}})^2}
    -N_i(t), 0
    \right\}.
\]
For Thompson Sampling, by adapting the standard finite-time analysis of Thompson Sampling \citep{agrawal2017near}, one obtains that, with probability at least $1-\frac{\delta}{K}$, $N_i(T)-N_i(t) \le C \frac{\log(KT/\delta)}{(g_i^F)^2}$,
for a universal constant $C>0$. Hence we may take
\[
    B_i^{\mathrm{TS}}(t,T,\delta/K)
    =
    C\frac{\log(KT/\delta)}{(\mu_K^F-\mu_i^F)^2}.
\]
This bound follows from the usual TS regret argument: after arm $i$ has been sampled on the order of $\frac{\log(KT/\delta)}{(g_i^F)^2}$ times, its posterior sample exceeds the midpoint $\frac{\mu_i^F+\mu_K^F}{2}$ only with probability $O(\frac{\delta}{KT})$; the remaining pulls are controlled by the standard anti-concentration step for the target arm.

Substituting either $B_i^{\mathrm{UCB}}$ or $B_i^{\mathrm{TS}}$ into \Cref{thm:general-attack-cost} gives a direct but coarser guarantee. The algorithm-specific UCB and Thompson Sampling attacks are sharper because they suppress the arm below a tailored lower confidence or posterior-sampling threshold and then use the exponential suppression lemmas, which eliminate the extra post-suppression term.

\subsection{Proofs of the General Suppression Framework}\label{app:general-proofs}

\begin{proof}[Proof of \Cref{lem:general-suppression}]
Let $\mathcal G_i (t, T, \frac{\delta}{K})$ be the event in \Cref{def:suppression-bounded-algorithm} for arm $i$ with confidence parameter $\frac{\delta}{K}$.
Fix an outcome in $\mathcal E\cap \mathcal G_i (t, T, \frac{\delta}{K})$, and write $B_i \equiv B_i(t,T,\frac{\delta}{K})$.
Immediately after the injection, arm $i$ has $N_i^0(t)+n_i$ samples in the learner's history. We will show that, as long as arm $i$ has been pulled no more than $B_i$ additional times after this injection, its empirical mean stays below the target-arm empirical mean. The post-suppression pull bound in \Cref{def:suppression-bounded-algorithm} then applies.

Let $m \triangleq N_i(s)-N_i(t)$ for some round $s \in \{t, t+1,\dots,T\}$, where $N_i(t)$ denotes the count after the $n_i$ fake samples have been injected. Since no fake samples are injected into arm $i$ after round $t$, all samples added after round $t$ are genuine. On $\mathcal E$, every genuine empirical mean of arm $i$ is upper bounded by $\mu_i+\beta(\cdot)$; in particular, the combined empirical mean of the first $N_i^0(t)+m$ genuine samples of arm $i$ is at most $\mu_i+\beta(N_i^0(t)+m) \le \mu_i + \beta (N_i^0(t)) \le \hat{\mu}^0_i(t)+2\beta(N_i^0(t)) = U_i(t)$, where we used that $\beta(\cdot)$ is nonincreasing and $\mathcal E$ also gives $\mu_i\le \hat{\mu}_i^0(t)+\beta(N_i^0(t))$. Hence
\begin{equation}
\label{eq:general-suppression-mean}
    \hat{\mu}_i(s)
    \le
    \frac{
        U_i(t)\bigl(N_i^0(t)+m\bigr)+ n_i \rl
    }{
        N_i^0(t)+n_i+m
    }.
\end{equation}

Now suppose $m \le B_i$. Since \Cref{eq:general-suppression-condition} holds and $\rl < L_K(t)$, we have
\[
    n_i(L_K(t)-\rl) > (U_i(t) - L_K(t)) (N_i^0(t)+B_i)
    \ge (U_i(t) - L_K(t))(N_i^0(t)+m).
\]
Equivalently,
    $U_i(t)(N_i^0(t)+m) + n_i \rl < L_K(t)(N_i^0(t)+n_i+m)$.
Combining this with \Cref{eq:general-suppression-mean} yields $\hat{\mu}_i(s) < L_K(t)$. On $\mathcal E$, for every $s\ge t$,
\[
    \hat{\mu}_K(s)
    \ge \mu_K-\beta(N_K^0(s))
    \ge \mu_K-\beta(N_K^0(t))
    \ge \hat{\mu}_K^0(t)-2\beta(N_K^0(t))
    =L_K(t).
\] 
Therefore, $\hat\mu_i(s) < \hat\mu_K(s)$ holds at every subsequent round $s$. Applying this argument recursively, we conclude that $N_i(T) - N_i(t) \le B_i$.
This proves the single-arm claim on $\mathcal E\cap \mathcal G_i(t,T,\delta/K)$, which has probability at least $1-\delta-\delta/K$. Applying the same argument to all non-target arms and using a union bound over $\mathcal E$ and the events $\mathcal G_i(t,T,\delta/K)$ gives
\[
    \mathbb P\left(
    \mathcal E\cap \bigcap_{i=1}^{K-1} \mathcal G_i (t, T, \delta/K)
    \right)
    \ge 1-\delta-\frac{K-1}{K}\delta
    \ge 1-2\delta.
\]
Thus all non-target arms are suppressed simultaneously with probability at least $1-2\delta$.
\end{proof}

\begin{proof}[Proof of \Cref{thm:general-attack-cost}]
For each non-target arm $i$, let $t_i$ be the first round at which the attacker suppresses arm $i$. Since $\rl<\mu_K-3\beta(1)$, on $\mathcal E$ we have $\rl<L_K(t_i)$. Write $B_i^\delta(t_i) = B_i \left(t_i, T, \frac{\delta}{K} \right)$ and $M_i(t_i) = N_i^0(t_i) + B_i^\delta(t_i)$. The attacker injects
\[
    n_i = \frac{U_i(t_i)-L_K(t_i)}{L_K(t_i)-\rl} M_i(t_i)+1
\]
fake samples of arm $i$, each with reward value $\rl$. The choice of $n_i$ ensures that the sufficient condition in \Cref{eq:general-suppression-condition} holds with confidence parameter $\frac{\delta}{K}$.
We first bound the number of injected samples explicitly. The assumption $\rl < \mu_K-3\beta(1)$ implies $\rl < \mu_K-3\beta(N_K^0(t_i))$ for every trigger time. On $\mathcal E$, $U_i(t_i) \le \mu_i+3\beta(N_i^0(t_i))$ and $L_K(t_i) \ge \mu_K-3\beta(N_K^0(t_i))$.
Therefore,
\[
    \frac{U_i(t_i)-\rl}{L_K(t_i)-\rl} \le
    \frac{\mu_i+3\beta(N_i^0(t_i))-\rl}{\mu_K-3 \beta(N_K^0(t_i)) - \rl} 
    \le 
    \frac{\mu_i+3\beta(1)-\rl}{\mu_K-3\beta(1)-\rl} ,
\]
and $\frac{U_i(t_i)-L_K(t_i)}{L_K(t_i)-\rl} \le \frac{\Delta_i+6\beta(1)}{\mu_K-3\beta(1)-\rl}$.
Therefore $n_i \le \frac{\Delta_i + 6\beta(1)}{\mu_K - 3 \beta(1) - \rl} M_i(t_i) + 1$.
Consequently,
\begin{align*}
    N^F(T)
    &=\sum_{i=1}^{K-1} N_i^F(T)\le\sum_{i=1}^{K-1}n_i \le \sum_{i=1}^{K-1}\left[
    \frac{\Delta_i+6\beta(1)}{\mu_K - 3 \beta(1) - \rl} M_i(t_i) + 1
    \right].
\end{align*}

Next, by \Cref{lem:general-suppression}, on the event
    $\mathcal E \cap \bigcap_{i = 1}^{K - 1} \mathcal G_i \left(t_i, T, \frac{\delta}{K} \right)$,
each non-target arm $i$ is selected at most $B_i^\delta(t_i)$ additional times after its injection. Before the injection, arm $i$ contributes $N_i^0(t_i)$ genuine pulls, and the injection itself contributes $n_i$ fake pulls to the learner's history. Thus the total number of non-target-arm pulls in the learner's history is at most $\sum_{i = 1}^{K - 1} (M_i(t_i)+n_i)$, and
\[
    M_i(t_i) + n_i \le \frac{U_i(t_i)-\rl}{L_K(t_i)-\rl} M_i(t_i) + 1.
\]
This gives the exact-statistic lower bound on the number of target-arm selections. Since $\mathbb P(\mathcal E) \ge 1 - \delta$ and each $\mathcal G_i(t_i, T, \frac{\delta}{K})$ fails with probability at most $\frac{\delta}{K}$, a union bound gives success probability at least $1 - 2 \delta$.
Substituting the first ratio bound into the exact-statistic pull bound gives the displayed lower bound on target-arm selections.
The cumulative reward-deviation cost is:
\[
    C^F(T)
    =\sum_{i=1}^{K-1}(\mu_i-\rl)N_i^F(T)
    \le\sum_{i=1}^{K-1}(\mu_i-\rl)\left[
    \frac{\Delta_i + 6\beta(1)}{\mu_K - 3\beta(1) - \rl}M_i(t_i)+1
    \right].
\]
This yields the displayed cost bound without a separate concentration argument.
Because the argument holds for any feasible choice of $t_i$, the attacker may choose the trigger times that minimize the resulting upper bound.
\end{proof}

\section{Additional Experiments}
\label{app:ts-experiments}

\begin{figure}[htbp]
    \centering
    \begin{subfigure}[t]{0.48\linewidth}
        \vspace{0pt}
        \centering
        \fdiPanel{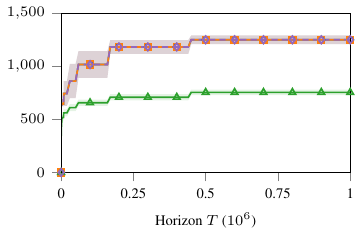}
        \caption{Attack cost vs.\ $T$}
        \label{fig:ts-cost-horizon}
    \end{subfigure}
    \hfill
    \begin{subfigure}[t]{0.48\linewidth}
        \vspace{0pt}
        \centering
        \fdiPanel{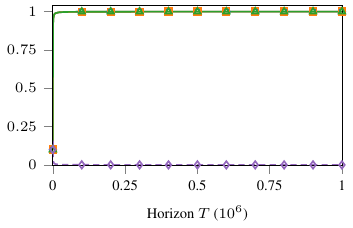}
        \caption{Target-arm ratio vs.\ $T$}
        \label{fig:ts-selection-ratio}
    \end{subfigure}
    \fdiLegend
    \caption{TS experiments with 20 repetitions. Shading shows one sample standard deviation.}
    \label{fig:ts}
\end{figure}

TS follows the shared settings in \Cref{sec:exp}, with arm means $(0.9,0.85,\ldots,0.45)$ and the smallest-mean arm as target. Genuine rewards are Gaussian with $\sigma =0.05$; posterior samples have variance $\frac{1}{N_i(t)}$. We use $\delta_0=1$.
At $T=10^6$, SBI and PBI reach $99.99\%$ target-arm selection ratio with an average of $17.65$ fake samples, while the matched-cost heuristic reaches $0.01\%$. LI reaches $99.99\%$ with $9$ samples.

\end{document}